\documentclass{article}

\usepackage{subfiles}

\usepackage[plain]{shortcut}

\def\BibFile{../../post_training.bib}
\graphicspath{{../../figures/}}
\def\TikzLocation{../tikz/}

\usepackage{enumitem}
\usepackage{multirow}

\usepackage{dblfloatfix}

\title{Post Training in Deep Learning}
\author{Thomas Moreau \hskip5em Julien Audiffren\\[.5em]
	CMLA, ENS Paris-Saclay, CNRS, \\
	Universit\'e Paris-Saclay,\\
	94235 Cachan, France\\
	\texttt{\{thomas.moreau, julien.audiffren\}@cmla.ens-cachan.fr}\\
}

\usepackage{xr}
\def\crossref{\externaldocument{/home/tom/.cache/gummi/.0_last_kernel.tex}}

\begin{document}

\def\crossref{}
\def\biblio{}


%\author{Thomas Moreau\inst{1} \and Julien Audiffren\inst{1}}
%
%\authorrunning{T. Moreau et. al.} % abbreviated author list (for running head)
%\institute{Centre de Mathematiques et Leurs Applications\\
%ENS Paris-Saclay, CNRS, Université Paris-Saclay,\\
%94235, Cachan, France.}

\maketitle

\begin{abstract}

	One of the main challenges of deep learning methods is the choice of an appropriate
	training strategy. In particular, additional steps, such as unsupervised pre-training,
	have been shown to greatly improve the performances of deep structures.
	In this article, we propose an extra training step, called post-training, which
	only optimizes the last layer of the network. We show that this procedure can be
	analyzed in the context of kernel theory, with the first layers computing an
	embedding of the data and the last layer a statistical model to solve
	the task based on this embedding. This step makes sure that the embedding, or
	representation, of the data is used in the best possible way for the considered
	task. This idea is then tested on multiple architectures with various data sets,
	showing that it consistently provides a boost in performance.

\end{abstract}


\crossref{}

\section{Training Neural Networks}
\label{sec:posttrain:context}

%% Motivation
	One of the main challenges of the deep learning methods is to efficiently
	solve the highly complex and non-convex optimization problem involved in the
	training step. Many parameters influence the performances of trained networks,
	and small mistakes can drive the algorithm into a sub-optimal local minimum,
	resulting into poor performances \citep{bengio2007scaling}. Consequently, the
	choice of an appropriate training strategy is critical to the usage of deep
	learning models.
	
%% SGD
	The most common approach to train deep networks is to use the stochastic gradient
	descent (SGD) algorithm. This method selects a
	few points in the training set, called a batch, and compute the gradient of a cost
	function relatively to all the layers parameter. The gradient is then used to
	update the weights of all layers. Empirically, this method converges most of the
	time to a local minimum of the cost	function which have good generalization
	properties. The stochastic updates estimate the gradient of the error on the input
	distribution, and several works proposed to use variance reduction technique such
	as Adagrap \citep{Duchi2011}, RMSprop \citep{Hinton2012} or Adam \citep{Kingma2015},
	to achieve faster convergence.
	
%% Pre training
	While these algorithms converge to a local minima, this minima is often influenced
	by the properties of the initialization used for the network weights. A frequently
	used approach to find a good starting point is to use pre-training 
	\citep{larochelle2007empirical,Hinton2006,hinton2006reducing}. This method
	iteratively constructs each layer using unsupervised learning to capture the
	information from the data. The network is then fine-tuned using SGD to solve the
	task at hand. Pre-training strategies have been applied successfully to many
	applications, such as classification tasks \citep{bengio2007scaling, poultney2006efficient},
	regression \citep{hinton2008using}, robotics \citep{hadsell2008deep} or information
	retrieval \citep{salakhutdinov2009semantic}. The influence of different pre-training
	strategies over the different layers has been thoroughly studied in
	\citet{larochelle2009exploring}.
%% Role of the layers
	In addition to improving the training strategies, these works also shed light onto
	the role of the different layers \citep{erhan2010does, montavon2011kernel}. The first
	layers of a deep neural network, qualified as \textit{general}, tend to learn feature
	extractors which can be reused in other architectures, independently of the solved task.
	Meanwhile, the last layers of the network are much more dependent of the task and
	data set, and are said to be \textit{specific}.

%% Kernel and deep learning
%	While Deep Learning achieves better results than shallow structures, the later
%	are generally easier to train and more stable. Multiple analysis have highlighted
%	connections between these two fields and some ideas from shallow models have been
%	adapted to deep models. For instance, ideas from kernel methods have been used to
%	introduce novel regularization functions \citep{yu2009deep}, or to study empirically
%	the layer-wise evolution of the representation of the data in deep learning networks
%	\citep{montavon2011kernel}. Reciprocally, \cite{zhuang2011two} used ideas from the
%	structure of neural network to propose a multilayer multiple kernels learning
%	algorithm. Another impact of the deep learning on kernel methods is arguably the
%	design of new type of kernels, derived from the computation graphs of neural
%	networks \citep{cho2009kernel}.

%% Separation representation / model
	Deep Learning generally achieves better results than shallow structures, but the later
	are generally easier to train and more stable. For convex models such as logistic
	regression, the training problem is also convex when the data representation is fixed.
	The separation between the representation and the model learning is a key ingredient
	for the model stability. When the representation is learned simultaneously, for
	instance with dictionary learning or with EM algorithms, the problem often become
	non-convex. But this coupling between the representation and the model is critical for
	end-to-end models. For instance, \citet{Hinton2006} showed that for networks trained
	using pre-training, the fine-tuning step -- where all the layers are trained together
	-- improves the performances of the network. This shows the importance of the adaptation
	of the representation to the task in end-to-end models.

%% Contribution
	Our contribution in this chapter is an additional training step which improves the
	use of the representation learned by the network to solve the considered task. This
	new step is called \textit{post-training}. It is based on the idea of separating
	representation learning and statistical analysis and it should be used after the
	training of the network. In this step, only the specific layers are trained. Since
	the general layers -- which encode the data representation -- are fixed, this step
	focuses on finding the best usage of the learned representation to solve the desired
	task. In particular, we chose to study the case where only the last layer is trained
	during the post-training, as this layer is the most specific one
	\citep{yosinski2014transferable}. In this setting, learning the weights of the last
	layer corresponds to learning the weights for the kernel associated to the feature
	map given by the previous layers. The post-training scheme can thus be interpreted
	in light of different results from kernel theory. To summarize our contributions:
	\begin{itemize}
		\renewcommand{\labelitemi}{{\scriptsize$\bullet$}}
		\item We introduce a post-training step, where all layers except the last one are
		frozen. This method can be applied after any traditional training scheme for
		deep networks. Note that this step does not replace the end-to-end training,
		which co-adapts the last layer representation with the solver weights, but it
		makes sure that this representation is used in the most efficient way for the
		given task.
		\item We show that this post-training step is easy to use, that it can be effortlessly
		added to most learning strategies, and that it is computationally inexpensive.
		\item We highlight the link existing between this method and the kernel techniques.
		We also show numerically that the previous layers can be used as a kernel map when
		the problem is small enough.
		\item We experimentally show that the post-training does not overfit and often
		produces improvement for various architectures and data sets.
	\end{itemize}

	The rest of this article is organized as follows: \autoref{sec:post_train:contrib} introduces the
	post-training step and discusses its relation with kernel methods. \autoref{sec:post_train:numerical}
	presents our numerical experiments with multiple neural network architectures and
	data sets and \autoref{sec:post_train:discussion} discusses these results.

% section posttrain:context (end)

%%%

\crossref{}

\section{Post-training}
\label{sec:post_train:contrib}

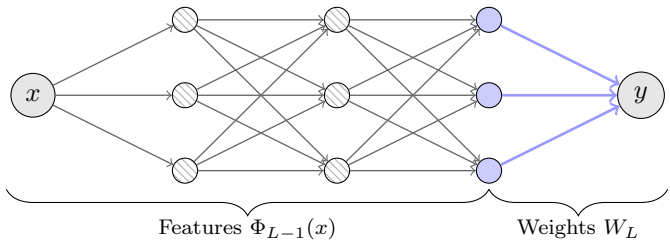
\begin{figure}[t]
	\centering
	\scalebox{1.0}{

	\begin{tikzpicture}
	
		\pgfmathsetmacro{\nwid}{2}
		\pgfmathsetmacro{\nlay}{3}
		\pgfmathtruncatemacro{\nlay}{\nlay-1}
		\pgfmathtruncatemacro{\nlaymoins}{\nlay-1}
		%\pgfmathsetmacro{\cgn}{gray!20}
		\node[draw, circle, fill=gray!20] (inp) at (-2,\nwid/2) {$x$};
		\node[draw, circle, fill=gray!20] (out) at (\nlay*2+2,\nwid/2) {$y$};\
		\definecolor{mycolor}{rgb}{.5,.5,.5}
		\foreach \i in {0,1,...,\nlay}{
			\ifthenelse{\NOT \nlay = \i}{}{}
			\foreach \j in {0,1,...,\nwid}{
				\ifthenelse{\NOT \nlay = \i}{
					\node[draw, circle, pattern=north west lines, pattern color=mycolor!60] (neuron-\i-\j) at (\i*2,\j*1.) {};}{
					\node[draw, circle, fill=blue!20] (neuron-\i-\j) at (\i*2,\j*1.) {};}
			}
		}
		\foreach \ij in {0,1,...,\nwid}{
		 	\draw[black!60, line width = .5 pt,->] (inp) to (neuron-0-\ij);
		 	\draw[blue!40, line width = 1 pt,->] (neuron-\nlay-\ij) to (out);
		}
		\foreach \i in {0,1,...,\nlaymoins}{
			\pgfmathtruncatemacro{\ii}{\i+1}
			\foreach \j in {0,1,...,\nwid}{
				\foreach \jj in {0,1,...,\nwid}{
					\draw[black!60, line width = .5pt,->] (neuron-\i-\j) to (neuron-\ii-\jj);
				}
			}
		}
		
		\draw [decoration={brace,amplitude=1em, mirror},decorate] ($(inp)-(0, \nwid/2)+(-1em,-.7em)$) -- ($(neuron-\nlay-0)+(0em,-.7em)$) node [black,midway,yshift=-1.5em] {\footnotesize Features $\Phi_{L-1}(x)$};
		\draw [decoration={brace,amplitude=1em, mirror},decorate] ($(neuron-\nlay-0)+(0em,-.7em)$) -- ($(out)-(0, \nwid/2)+(1em,-.7em)$) node [black,midway,yshift=-1.5em] {\footnotesize Weights $W_L$ };
	\end{tikzpicture}
	}
	\caption{Illustration of post-training applied to a neural network. During the post-training, only the weights of the blue edges are updated. The blue nodes can be seen as the embedding of $x$ in the feature space $\Xset_L~$. }
	\label{fig:lastkernel}
\end{figure}

%\subsection{\algo{}}

%The analysis of this fine tuning in light of the kernel theory makes it possible to reuse many results from this field to analyse its behavior.

In this section, we consider a feedforward neural network with $L$ layers, where $\Xset_1, \ldots, \Xset_{L}$ denote the input space of the different layers, typically $\Rset^{d_l}$ with $d_l > 0$ and \mbox{$\Yset = \Xset_{L+1}$} the output space of our network.
Let $\phi_{l} : \Xset_l \mapsto \Xset_{l+1}$ be the applications which respectively compute the output of the $l$-th layer of the network, for $1\le l \le L$, using the output of the $\smeq l-1$-th layer and $\Phi_L = \phi_L \circ \dots \circ \phi_1$ be the mapping of the full network from $\Xset_1$ to $\Yset~.$ Also, for each layer $l$, we denote $\pmb W_l$ its weights matrix and $\psi_l$ its activation function.

The training of our network is done using a convex and continuous loss function $\ell : \Yset \times \Yset \mapsto \Rset^+$. The objective of the neural network training is to find weights parametrizing $\Phi_L$ that solves the following problem:
\begin{equation}
	\min_{\Phi_L} \mathbb E_{(x, y) \sim \mathcal P}\left[\ell\left(\Phi_L(x), y\right)\right]~.
	\label{eq:init}
\end{equation}
for a certain input distribution $\mathcal P$ in $(\Xset_1,\Yset)$. The training set is
\mbox{$\Dset = \left(x_i, y_i\right)_{i=1}^N$}, drawn from this input distribution.

Using these notations, the training objective \autoref{eq:init} can then be rewritten
\begin{equation}
	\label{eq:last_layer}
	\min_{\Phi_{L-1}, \pmb W_L} \mathbb E_{(x, y) \sim \mathcal P}
		\left[\ell\left(\psi_L\left(\pmb W_L\Phi_{L-1}(x)\right),~y\right)\right]~.
\end{equation}
This reformulation highlights the special role of the last layer in our network compared to the others.
When $\Phi_{L-1}$ is fixed, the problem of finding $\pmb W_L$ is simple for several popular choices
of activation function $\psi_L$ and loss $\ell~.$
For instance, when the activation function $\psi_L$
is the {\tt softmax} function and the loss $\ell$ is the {\tt cross entropy},
\autoref{eq:last_layer} is a multinomial logistic regression. In this case, training the last layer is equivalent to a regression of the labels $y$
using the embedding of the data $x$ in $\Xset_L$ by the mapping $\Phi_{L-1}~$. Since the problem is convex in $\pmb W_L$ (see \autoref{sec:post_train:appendix}), classical optimization techniques can efficiently produce  an accurate approximation of
the optimal weights $\pmb W_L$ -- and this optimization given the mapping  $\Phi_{L-1}$ is the idea behind post-training.

Indeed, during the regular training, the network tries to simultaneously learn suitable representation for the data in the space $\Xset_L$ through its $L-1$ first layer and the best use of this representation with $\pmb W_L$.
This joint minimization is a strongly non-convex problem, therefore resulting in a potentially sub-optimal usage of the learned data representation.

The post-training is an additional step of learning which takes place after the regular training and proceeds as follows :
\begin{enumerate}\itemsep.5em
	\item {\bf Regular training:} This step aims to obtain interesting features
	to solve the initial problem, as in any usual deep learning training. Any
	training strategy can be applied to the network, optimizing the empirical loss
	\begin{equation}
		\label{eq:loss}
		\argmin_{\Phi_L} \frac{1}{N}\sum_{i=1}^N \ell\left(\Phi_L(x_i), y_i\right)~.
	\end{equation}
	The stochastic gradient descent explores the parameter space and provides
	a solution for $\Phi_{L-1}$ and $\pmb W_L$. This step is non restrictive: any type of training strategy can be used
	here, including gradient bias reduction techniques, such as Adagrad
	\citep{Duchi2011}, or regularization strategies, for instance
	using Dropout \citep{dahl2013improving}. Similarly, any type of stopping
	criterion can be used here. The training might last for a fixed number
	of epochs, or can stop after using early stopping \citep{morgan1989generalization}.
	Different combinations of training strategies and stopping criterion are
	tested in \autoref{sec:post_train:numerical}.

	\item {\bf Post-training:} During this step, the first $L-1$ layers are fixed and only the last layer of
	the network, $\phi_L,$ is trained by minimizing over $\pmb W_L$ the following problem
	\begin{equation}\label{eq:post_training}
	\argmin_{\pmb W_L}\frac{1}{N} \sum_{i=1}^N \ltil\left( \Phi_{L-1}(x_i) \pmb W_L\tran,y_i \right) + \lambda \|\pmb W_L\|^2_2~,
	\end{equation}
	where $\ltil(x,y) := \ell ( \psi_L(x),y)~.$
	This extra learning step uses the mapping $\Phi_{L-1}$ as an embedding of the data
	in $\Xset_L$ and learn the best linear predictor in this space. This optimization
	problem takes place in a significantly lower dimensional space and since there is
	no need for back propagation, this step is computationally faster. To reduce the risk
	of overfitting with this step, a $\ell_2$-regularization is added.
	%This additional regularization also makes the problem strongly convex for $\lambda > 0$ in most of the classical network architecture where \autoref{eq:post_training} is convex.
	\autoref{fig:lastkernel} illustrates the post-training step.
\end{enumerate}

We would like to emphasize the importance of the $\ell_2$-regularization used during the post-training \autoref{eq:post_training}.
This regularization is added regardless of the one used in the regular training, and for all the network architectures.
The extra term improves the strong convexity of the minimization problem, making post-training more efficient, and promotes the generalization of the model.
The choice of the $\ell_2$-regularization is motivated from the comparison with the kernel framework discussed in \autoref{sec:post_train:kernel} and from our experimental results.
%, as
%the comparison  is motivated by classical results in machine learning and promotes the generalization of the model, ans is a natural choice notably in kernel theory, as

\vskip1em
\begin{remark}[Dropout.]
	It is important to note that Dropout should not be applied on the previous layers of the network during the post-training, as it would lead to changes in the feature function $\Phi_{L-1}$.
\end{remark}

%
%
%\paragraph{Convexity.}
%%
%For reasonable choices of activation and loss functions, \autoref{eq:post_training} is convex and thus can be minimized efficiently.
%As such, theoretical guarantees such as converge rates in $\bigO(\sqrt{1/N})$ apply for gradient descent technique combined with reasonable choices of learning rate and regularization \citep{nemirovski2009robust}.
%
%%Those results ensure that \algo{} will produce an improvement of the performances of the network.
%%
%

\section{Link with Kernels}
\label{sec:post_train:kernel}

	In this section, we show that for the case where $\Xset_{L}=\Rset^{d_L}$
	for some $d_L>0$ and $\Xset_{L+1}=\Rset$, $\pmb W_L^*$ can be approximated
	using kernel methods. We define the kernel $k$ as follows,
	\begin{equation*}
	\begin{split}
	k : \Xset_1 \times \Xset_1 &\mapsto \Rset \\
	 (x_1,x_2) &\rightarrow \left\langle\Phi_{L-1}(x_1), \Phi_{L-1}(x_2) \right\rangle~.
	\end{split}
	\end{equation*}
	Then $k$ is the kernel associated with the feature function $\Phi_{L-1}$. It is
	easy to see that this kernel is continuous positive definite and that for
	$\pmb W \in \Rset^{d_L}$, the function
	\begin{equation}
		\label{eq:post_train:gW}
		\begin{split}
		g_{\pmb W} : \Xset_{1} &\mapsto \Xset_{L+1} \\
		 x &\rightarrow \left\langle\Phi_{L-1}(x), \pmb W \right\rangle
		\end{split}
	\end{equation}
	belongs by construction to the Reproducing Kernel Hilbert Space (RKHS)
	$\Hset_k$ generated by $k$.
	The post-training problem \autoref{eq:post_training} is therefore related to
	the problem posed in the RKHS space $\Hset_k$, defined by 
	\begin{equation*}
		g^* = \argmin_{g \in \Hset_k}\frac{1}{N} \sum_{i=1}^N \ltil\left( g(x_i), y_i \right)
				+ \lambda \|g\|_{\Hset_k}^2~,
	\end{equation*}
	This problem is classic for the kernel methods. With mild hypothesis on
	$\ltil$, the generalized representer theorem can be applied
	\citep{scholkopf2001generalized}. As a consequence, there exists
	$\alpha^* \in \Rset^{N}$ such that
	\begin{equation}\label{eq:opt_g}
		\begin{split}
		g^* &:=\argmin_{g \in \Hset_k}\frac{1}{N} \sum_{i=1}^N \ltil\left( g(x_i), y_i \right)
					+ \lambda \|g\|_{\Hset_k}^2~\\
		&= \sum_{i=1}^N \alpha^*_i k(X_i,\cdot)
					= \sum_{i=1}^N \left\langle  \alpha^*_i  \Phi_{L-1}\left( x_i \right),
							\Phi_{L-1}(\cdot) \right\rangle.
		\end{split}
	\end{equation}
	Rewriting \autoref{eq:opt_g} with $g^*$ of the form \autoref{eq:post_train:gW},
	we have that $g^* = g_{\pmb W^*}$, with
	\begin{equation}\label{eq:w*}
		\pmb W^* = \sum_{i=1}^N \alpha^*_i  \Phi_{L-1}\left(x_i\right)`.
	\end{equation}
	We emphasize that $\pmb W^*$ gives the optimal solution for the problem \autoref{eq:opt_g}
	and should not be confused with $\pmb W_L^*~,$ the optimum of \autoref{eq:post_training}.
	However, the two problems differ only in their regularization, which are closely
	related (see the next paragraph). Thus $\pmb W^*$ can thus be seen as an approximation
	of the optimal value $\pmb W^*_L$. It is worth noting that in our experiments, $\pmb W^*$
	appears to be a nearly optimal estimator of  $\pmb W^*_L$ (see \autoref{sub:expopt}).

\paragraph{Relation between $\| \cdot\|_\Hset$ and $\|\cdot \|_2$.}

	The problems \autoref{eq:opt_g} and \autoref{eq:post_training} only differ in the
	choice of the regularization norm.
	By definition of the RKHS norm, we have
	\[
		\|g_W\|_\Hset = \inf \left \{ \|v\|_2 \middle/~~~\forall x \in \Xset_1,~~~
			\langle v,~ \Phi_{L-1}(x)\rangle = g_W(x)\right \}~.
	\]
	Consequently, we have that $\|g_W\|_\Hset \le \|W\|_2~,$ with equality when $\text{Vect}(\Phi_{L-1}(\Xset_1))$ spans the entire space $\Xset_{L}$. In this case, the norm induced by the RKHS is equal to the $\ell_2$-norm.
	This is generally the case, as the input space is usually in a far higher dimensional space than the embedding space, and since the neural network structure generally enforces the independence of the features. Therefore, while both norms can be used in \eqref{eq:post_training}, we chose to use the $\ell_2$-norm for all our experiments as it is easier to compute than the RKHS norm.

\paragraph{Close-form Solution.}
In the particular case where $\ell(y_1,y_2) = \| y_1 - y_2 \| ^2$ and $f(x)=x$, \autoref{eq:opt_g} can be reduced to a classical Kernel Ridge Regression problem.
In this setting, $W^*$ can be computed by combining \eqref{eq:w*} and
\begin{equation}
	\label{eq:sol_alpha}
	\alpha^* = \left(\Phi_{L-1}(\Dset)\tran \Phi_{L-1}(\Dset) + \lambda {\pmb I}_N \right)^{-1} Y~,
\end{equation}
where $\Phi_{L-1}(\Dset) = \begin{bmatrix}
	\Phi_{L-1}(x_1), \dots \Phi_{L-1}(x_N)
\end{bmatrix}$ represents the matrix of the input data $\left \{ x_1, \dots x_N \right \}$ embedded in $\Xset_{L}$, $Y$ is the matrix of the output data $\left \{ y_1, \dots, y_N \right \}$ and ${\pmb I}_N$ is the identity matrix in $\Rset^N$.
This result is experimentally illustrated in  \autoref{sub:expopt}.
Although  data sets are generally too large for \autoref{eq:sol_alpha} to be computed in practice, it is worth noting that some kernel methods, such as Random Features \citep{rahimi2007random}, can be applied to compute approximations of the optimal weights during the post-training.

\paragraph{Multidimensional Output.}
Most of the previously discussed results related to kernel theory hold for multidimensional output spaces, \ie{} dim$(\Xset_{L+1})=d>1$, using multitask or operator valued kernels \citep{kadri2015operator}.
Hence the previous remarks can be easily extended to multidimensional outputs, encouraging the use of post-training in most settings.

%%%%

\crossref{}

\section{Experimental Results}
\label{sec:post_train:numerical}

This section provides numerical arguments to study
post-training  and its influence on performances,
over different data sets and network architectures.
All the experiments were run using {\tt python}
and {\tt Tensorflow}. The code to reproduce the
figures is available online\footnote{The code is available at
\url{https://github.com/tomMoral/post_training}
}. The results of all the experiments are discussed in
depth in \autoref{sec:post_train:discussion}.

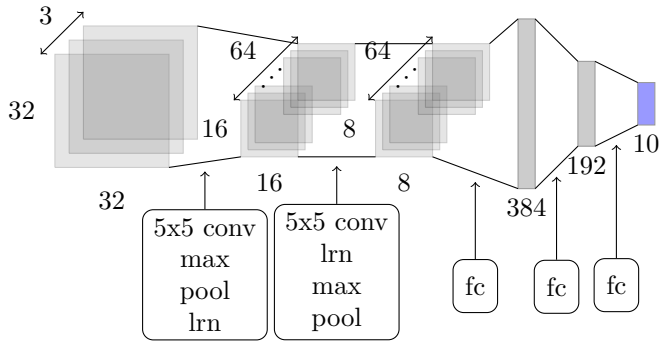
\begin{figure}[t]
	\centering
	\fontsize{10}{12}\selectfont
	\scalebox{1}{

	\begin{tikzpicture}
		
		\definecolor{mycolor}{rgb}{.5,.5,.5}
		\tikzset{
		%Define style for boxes
		varstyle/.style={
		       rectangle,
		       fill=white,
		       rounded corners,
		       draw=black,
		       text width=4em,
		       minimum height=2em,
		       text centered},
		%Define style for boxes
		fcstyle/.style={
		       rectangle,
		       fill=white,
		       rounded corners,
		       draw=black,
		       text width=1em,
		       minimum height=2em,
		       text centered}
		}
		% layer 1
		\pgfmathsetmacro{\len}{1.5}
		\pgfmathsetmacro{\off}{\len/8}
		\pgfmathsetmacro{\nfilt}{3}
		\pgfmathsetmacro{\x}{0}
		\pgfmathsetmacro{\xstart}{\x-\len*\nfilt/16}
		\pgfmathsetmacro{\ystart}{-\len/2-\len*\nfilt/16}
		
		\foreach \j in {1,...,\nfilt}{
			\draw[fill=mycolor,opacity=0.2,draw=black] (\xstart+\j*\off,\ystart+\j*\off)
				-- +(\len,0) -- +(\len,\len) -- +(0,\len) -- +(0,0);	
		}
		\node at (\xstart-.25,\ystart+\len/2+\off){32};
		\node at (\xstart+\len/2+\off,\ystart-.25){32};
		\draw[<->] (\xstart,\ystart+\len+\off)-- node [above left] {\nfilt} ++(\off*\nfilt, \off*\nfilt) ;
		
		%layer 2
		\pgfmathsetmacro{\nnfilt}{64}
		\pgfmathsetmacro{\nnnfilt}{sqrt{\nnfilt}+1}
		\pgfmathsetmacro{\nx}{\x+\off*\nfilt/2+2.4}
		\pgfmathsetmacro{\nlen}{.75}
		\pgfmathsetmacro{\noff}{\nlen/8}
		\draw (\xstart+\nfilt*\off+\len, \ystart+\off*\nfilt+\len) -- (\nx+\noff*\nnnfilt/2, \nnnfilt/2*\noff+\nlen/2);
		\draw (\xstart+\len+\off, \ystart+\off) -- (\nx-\noff*\nnnfilt/2+\noff, -\nnnfilt/2*\noff-\nlen/2+\noff)  node [midway] (expl){};
		\draw[<-] (expl.south) -- ++(0, -.5) node [below, varstyle] (a) {5x5 conv\\max pool\\lrn};
		\pgfmathsetmacro{\x}{\nx}
		\pgfmathsetmacro{\nfilt}{sqrt{\nnfilt}+1}
		\pgfmathsetmacro{\len}{\nlen}
		\pgfmathsetmacro{\off}{\noff}
		\pgfmathsetmacro{\xstart}{\x-\off*\nfilt/2}
		\pgfmathsetmacro{\ystart}{-\len/2-\off*\nfilt/2}

		\foreach \j in {1,...,\nfilt}{
			%\draw[fill=blue,opacity=0.2,draw=black] (\xstart+\j*\off,\ystart+\j*\off)
			%	-- +(\len,0) -- +(\len,\len) -- +(0,\len) -- +(0,0);
			\ifthenelse{\NOT \j = 5 \AND \NOT \j = 4\AND \NOT \j = 6}{
				\draw[fill=mycolor,opacity=0.2,draw=black] (\xstart+\j*\off,\ystart+\j*\off)
			-- +(\len,0) -- +(\len,\len) -- +(0,\len) -- +(0,0);}{
				}
		}
		\node at (\xstart-.25,\ystart+\len/2+\off){16};
		\node at (\xstart+\len/2+\off,\ystart-.25){16};
		\draw[<->] (\xstart,\ystart+\len+\off)-- node [above left] {\nnfilt} node [below,sloped] {\dots} ++(\off*\nfilt, \off*\nfilt) ;
		%layer 3
		
		\pgfmathsetmacro{\nx}{\x+\off*\nfilt/2+1.35}
		\draw (\xstart+\nfilt*\off+\len, \ystart+\off*\nfilt+\len) -- (\nx+\off*\nfilt/2, \nfilt/2*\off+\len/2);
		\draw (\xstart+\len+\off, \ystart+\off) -- (\nx-\off*\nfilt/2+\off, -\nfilt/2*\off-\len/2+\off)  node [midway] (expl){};
		\draw[<-] (expl.south) -- ++(0, -.5) node [below, varstyle] (a) {5x5 conv\\ lrn \\ max pool};
		\pgfmathsetmacro{\x}{\nx}
		\pgfmathsetmacro{\nnfilt}{64}
		\pgfmathsetmacro{\nfilt}{sqrt{\nnfilt}+1}
		\pgfmathsetmacro{\off}{\len/8}
		\pgfmathsetmacro{\xstart}{\x-\off*\nfilt/2}
		\pgfmathsetmacro{\ystart}{-\len/2-\off*\nfilt/2}

		\foreach \j in {1,...,\nfilt}{
			%\draw[fill=blue,opacity=0.2,draw=black] (\xstart+\j*\off,\ystart+\j*\off)
			%	-- +(\len,0) -- +(\len,\len) -- +(0,\len) -- +(0,0);
			\ifthenelse{\NOT \j = 5 \AND \NOT \j = 4\AND \NOT \j = 6}{
				\draw[fill=mycolor,opacity=0.2,draw=black] (\xstart+\j*\off,\ystart+\j*\off)
			-- +(\len,0) -- +(\len,\len) -- +(0,\len) -- +(0,0);}{
				}
		}
		\node at (\xstart-.25,\ystart+\len/2+\off){8};
		\node at (\xstart+\len/2+\off,\ystart-.25){8};
		\draw[<->] (\xstart,\ystart+\len+\off)-- node [above left] {\nnfilt} node [below,sloped] {\dots} ++(\off*\nfilt, \off*\nfilt);
		
		% fully connected layers
		\pgfmathsetmacro{\nx}{\x+\off*\nfilt/2+1.5*\len}
		\draw (\xstart+\nfilt*\off+\len, \ystart+\off*\nfilt+\len) -- (\nx, 3*\len/2);
		\draw (\xstart+\len+\off, \ystart+\off) -- (\nx, -3*\len/2) node [midway] (t){}  node [midway] (expl){};
		\draw[<-] (expl.south) -- ++(0, -1.02) node [below, fcstyle] {fc};
		\pgfmathsetmacro{\x}{\nx}
		\pgfmathsetmacro{\len}{3*\len}
		\pgfmathsetmacro{\width}{.1*\len}
		\draw[fill=mycolor,opacity=0.4,draw=black] (\x,-\len/2)
			-- +(\width,0) -- +(\width,\len) -- +(0,\len) -- +(0,0);
		\node at (\x+\width/2,-\len/2-.25){384};
		\pgfmathsetmacro{\nx}{\x+.35*\len}
		\pgfmathsetmacro{\len}{\len/2}
		\draw (\x+\width, \len) -- (\nx, \len/2);
		\draw (\x+\width, -\len) -- (\nx, -\len/2)  node [midway] (expl){};
		\draw[<-] (expl.south) -- ++(0, -1.1) node [below, fcstyle] (a) {fc};
		\pgfmathsetmacro{\x}{\nx}
		\draw[fill=mycolor,opacity=0.4,draw=black] (\x,-\len/2)
			-- +(\width,0) -- +(\width,\len) -- +(0,\len) -- +(0,0);
		\node at (\x+\width/2,-\len/2-.25){192};
		\pgfmathsetmacro{\nx}{\x+.7*\len}
		\pgfmathsetmacro{\len}{\len/2}
		\draw (\x+\width, \len) -- (\nx, \len/2);
		\draw (\x+\width, -\len) -- (\nx, -\len/2)  node [midway] (expl){};
		\draw[<-] (expl.south) -- ++(0, -1.5) node [below, fcstyle] (a) {fc};
		\pgfmathsetmacro{\x}{\nx}
		\draw[fill=blue,opacity=0.4,draw=black] (\x,-\len/2)
			-- +(\width,0) -- +(\width,\len) -- +(0,\len) -- +(0,0);
		\node at (\x+\width/2,-\len/2-.25){10};

	\end{tikzpicture}
	}
	\caption{Illustration of the neural network structure used for CIFAR-10. The last layer, represented in blue, is the one trained during the post-training. The layers are composed with classical layers: 5x5 convolutional layers (5x5 conv), max pooling activation (max pool), local response normalization (lrn) and fully connected linear layers (fc).}
	\label{fig:deconvolutional-layer}
\end{figure}

\subsection{Convolutional Neural Networks}
\label{sub:cifar10}

\begin{figure*}[t]
\centering
\includegraphics[width=\linewidth]{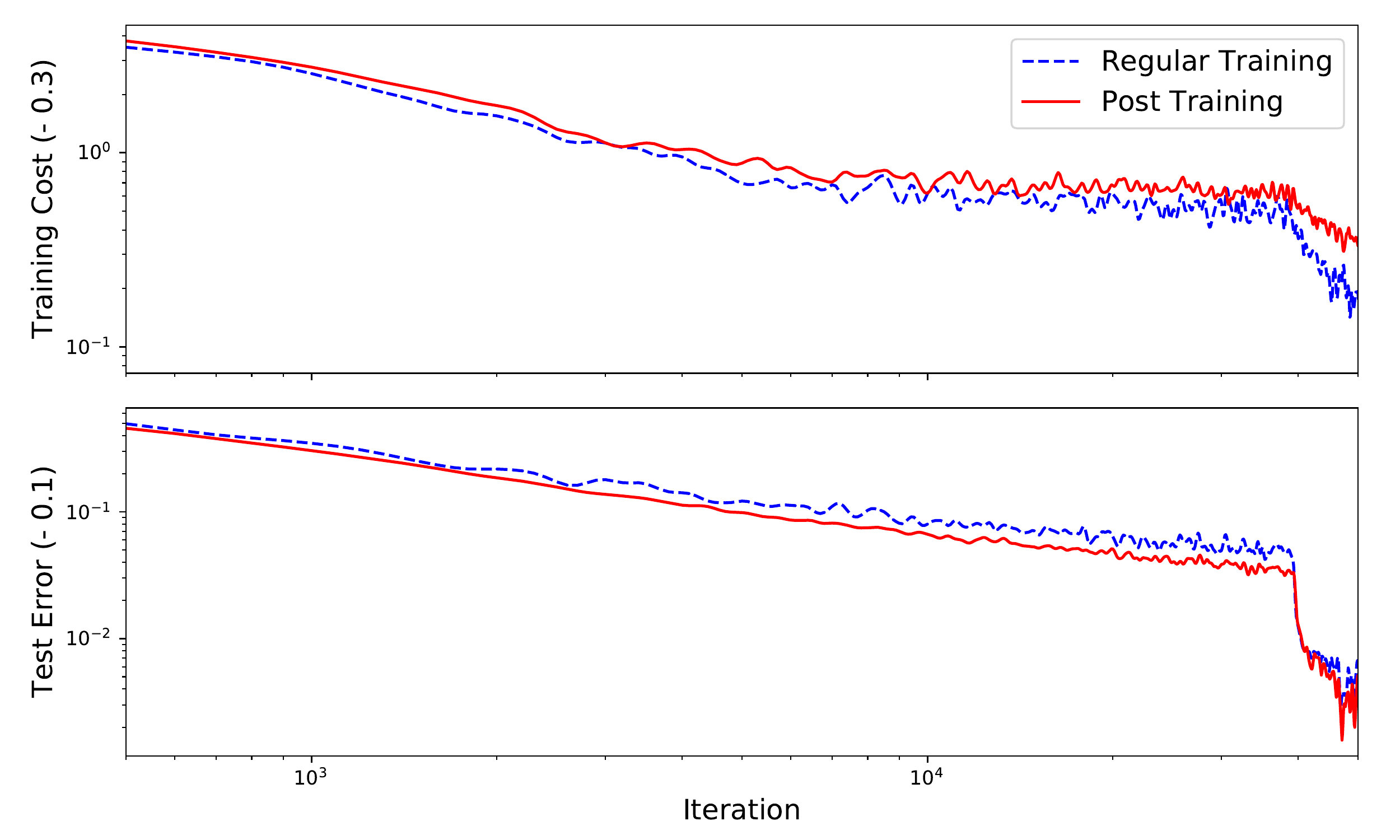}
\caption{Evolution of the performances of the neural network on the CIFAR10 data set,
		 ({\it dashed}) with the usual training and ({\it solid}) with the post-training
		 phase. For the post-training, the value of the curve at iteration $q$ is the error
		 for a network trained for $q-100$ iterations with the regular training strategy
		 and then trained for $100$ iterations with post-training. The \emph{top} figure
		 presents the classification error on the training set and the \emph{bottom} figure
		 displays the loss cost on the test set. The curves have been smoothed to increase
		 readability.
}
\label{fig:cifar10}
\end{figure*}

The post-training method can be applied easily to feedforward convolutional
neural network, used to solve a wide class of real world problems. To assert
its performance, we apply it to three classic benchmark datsets: CIFAR10
\citep{krizhevsky2009learning}, MNIST and FACES \citep{hinton2006reducing}.

\paragraph{CIFAR10.}

	This data set is composed of $60,000$ images $32\times 32$, representing objects from
	$10$ classes. We use the default architecture proposed by {\tt Tensorflow} for CIFAR10
	in our experiments, based on the original architecture proposed by
	\citet{krizhevsky2009learning}. It is composed of $5$ layers described in
	\autoref{fig:deconvolutional-layer}. The first layers use various common tools such as
	local response normalization (lrn), max pooling and RELU activation. The last layer have
	a \emph{softmax} activation function and the chosen training loss was the cross entropy
	function. The network is trained for $90k$ iterations, with batches of size $128$, using
	stochastic gradient descent (SGD), dropout and an exponential weight decay for the learning
	rate. \autoref{fig:cifar10} presents the performance of the network on the training and test
	sets for 2 different training strategies. The dashed line present the classic training with
	SGD, with performance evaluated every 100 iterations and the solid line present the
	performance of the same network where the last 100 iterations are done using post-training
	instead of regular training. To be clearer, the value of this curve at iteration $q$ is the
	error of the network, trained for $q-100$ iterations with the regular training strategy,
	and then trained for $100$ iterations with post-training. The regularization parameter
	$\lambda$ for post-training is set to $1\times 10^{-3}$.

	The results show that while the training cost of the network mildly increases due to the
	use of post-training, this extra step improves the generalization of the solution. The gain
	is smaller at the end of the training as the network converges to a local minimum, but it
	is consistent. Also, it is interesting to note that the post-training iterations 
	are $4\times$ faster than the classic iterations, due to their inexpensiveness.

\paragraph{Additional Data Sets.}
We also evaluate post-training on the MNIST data set
(65000  images $27 \times 27$, with 55000 for train and 10000 for test; 10 classes) and the pre-processed
FACES data set ($400$ images $64 \times 64$, from which $102400$ sub-images, $32 \times 32$, are extracted, with 92160 for training and 10240 for testing; 40 classes).
For each data set, we train two different convolutional
neural networks -- to assert the influence of the complexity
of the network over post-training:
\begin{itemize}
	\renewcommand{\labelitemi}{\scriptsize$\bullet$}
	\item a small network, with one convolutional layer
	($5 \times 5$ patches, $32$ channels), one $2\times 2$
	max pooling layer, and one fully connected hidden
	layer with $512$ neurons,
	%\item a medium network, with one convolutional layer ($5 \times 5$ patches, $32$ channels), one $2\times 2$ max pooling layer, one convolutional layer ($5 \times 5$ patches, $64$ channels), one $2\times 2$ max pooling layer and one fully connected hidden layer with $512$ neurons,
	\item a large network, with one convolutional layer
	($5 \times 5$ patches, $32$ channels), one $2\times 2$
	max pooling layer, one convolutional layer
	($5 \times 5$ patches, $64$ channels), one $2\times 2$
	max pooling layer and one fully connected hidden
	layer with $1024$ neurons.
\end{itemize}

We use dropout for the regularization, and set $ \lambda = 1\times 10^{-2}$.
We compare the performance gain resulting of the application of
post-training (100 iterations) at different epochs of each of
these networks. The results are reported in Table \ref{tab:conv_xp}.

\begin{center}
\begin{table}[t]
  \renewcommand{\arraystretch}{1.15}

\caption{Comparison of the performances (classification error)
of different networks on different data sets, at different epochs,
with or without post-training. }
\label{tab:conv_xp}
\centering
\begin{tabular}{|c|c|c|c|c|}
\hline
\mbox{Data set}&\mbox{ Network }&\mbox{Iterations}&\mbox{Mean (Std) Error in \% }&
				  \parbox{10em}{\vskip.3em\centering Mean (Std) Error with \\post-training in \% }\\
\hline
\multirow{9}{*}{FACES} & \multirow{3}{*}{Small} & 5000  & 21,5 (10) & 19,1 (12)\\
 && 10000  & 20 (4) & 19 (3,5)\\
 && 20000  & 18 (0,9) & 16,5 (0,8) \\\cline{2-5}
&\multirow{3}{*}{Large} & 5000 & 25 (15) & 24 (15)\\
 && 10000 & 15 (5) & 12 (5)\\
 && 20000 & 11 (0,5) & 10 (0,5)\\\hline

\multirow{6}{*}{MNIST} & \multirow{3}{*}{Small}
	 & 1000 & 10.7 (1)  & 9.2 (1,1)\\
	&& 2000 & 7,5 (0,7) & 6,7 (0,6)\\
 	&& 5000 & 4,1 (0,2) & 3,9 (0,2)\\\cline{2-5}
&\multirow{3}{*}{Large}
	 & 1000 & 9,1 (1,3) & 8,5 (1,4)\\
	&& 2000 & 4,1 (0,2) & 3,5 (0,2)\\
	&& 5000 & 1,1 (0,01) & 0,9 (0,01)\\
 \hline
\end{tabular}
\end{table}

\end{center}

As seen in \autoref{tab:conv_xp}, post-training improves the test
performance of the networks with as little as $100$ iterations
-- which is negligible compared to the time required to train
the network. While the improvement varies depending on the
complexity of the network, of the data set, and of the time
spent training the network, it is important to remark that
it always provides an improvement.

%The training cost did not drop as significantly as the test error with \algo{}, reducing the risk of overfitting the training set.
%Additional, \algo{} seems to reduce the variance of the error which originates from batches.

%Those results were reproduced with different architectures.
%We performed the same experiment with other neural networks structures, including networks with additional fully connected layers or with additional convolutional layers.
%Those results indicate that our observations are not tied to certain type of architecture.
%Additionally, the improvement of performances originating from \algo{} was constant, regardless of the regularization used during the training phase -- inluding dropout.
%It is worth noting that different values of $\lambda \in\left[10^{-6}, 10^1\right]$ were tested in this experiment.  The post-training process appeared to be robust to the choice of reasonable values of $\lambda$ (\ie{}  \mbox{$10^{-6} \le \lambda \le 10^{-2}$}).

% subsection cifar10 (end)

\subsection{Recurrent Neural Network}
\label{sub:ptb}

\begin{figure*}[t]
\centering
\includegraphics[width=\linewidth]{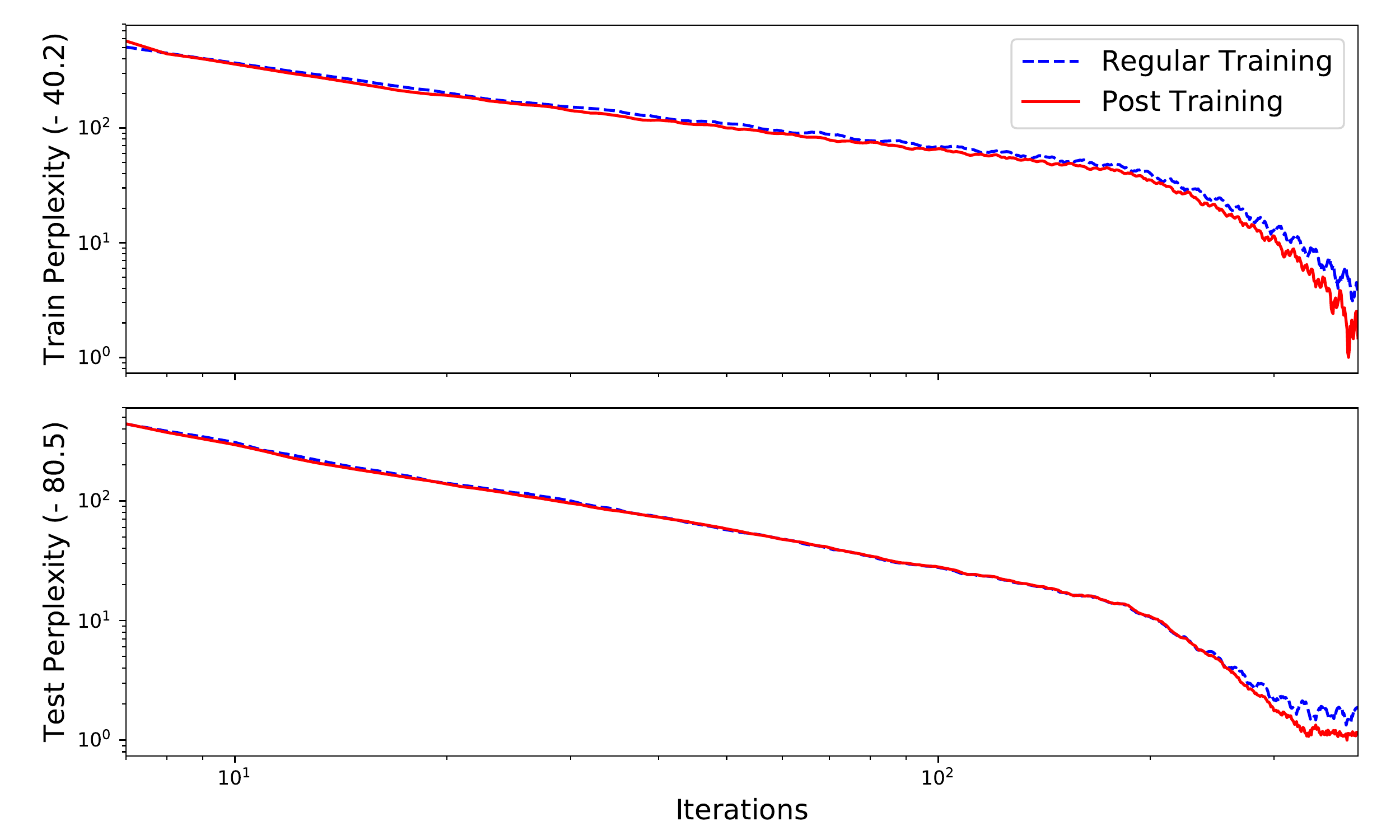}
\caption{Evolution of the performances of the Recurrent network on the PTB data set. The
		 \emph{top} figure presents the train perplexity and the \emph{bottom} figure
		 displays the test perplexity. For the post-training, the value of the curve at
		 iteration $q$ is the error for a network trained for $q-100$ iterations with the
		 regular training strategy and then trained for $100$ iterations with post-training.
}
\label{fig:ptb}
\end{figure*}

	While the kernel framework developed in \autoref{sec:post_train:contrib} does not apply directly
	to Recurrent Neural Network, the idea of post-training can still be applied.  In this
	experiment, we test the performances of post-training on Long Short-Term Memory-based
	networks (LSTM), using {\tt PTB} data set \citep{marcus1993building}.

\paragraph{Penn Tree Bank (PTB).}
This data set is composed of $929k$ training words and $82k$
test word, with a $10000$ words vocabulary. We train a
recurrent neural network to predict the next word given
the word history. We use the architecture proposed by
\citet{zaremba2014recurrent}, composed of 2 layers of $1500$
LSTM units with tanh activation, followed by a fully
connected {\tt softmax} layer. The network is trained
to minimize the average per-word perplexity for $100$ epochs,
with batches of
size $20$, using gradient descent, an exponential weight
decay for the learning rate, and dropout for regularization.
The performances of the network after each epoch are compared
to the results obtained if the $100$ last steps (i.e. $100$ batches)
are done using post-training. The regularization parameter
for post-training, $\lambda$, is set to $1\times 10^{-3}$.
The results are reported in \autoref{fig:ptb}, which presents
the evolution of the training and testing perplexity.

Similarly to the previous experiments, post-training improves
the test performance of the networks, even after the network
has converged.

\subsection{Optimal Last Layer for Deep Ridge Regression}
\label{sub:expopt}

	In this subsection we aim to empirically evaluate the
	close-form solution discussed in \autoref{sec:post_train:contrib}
	for regression tasks. We set the activation function of
	the last layer to be the identity $f_L(x) = x,$ and
	consider the loss function to be the least-squared
	error $\ell(x,y) = \| x-y \|^2_2$ in \autoref{eq:init}.
	In in each experiment, \autoref{eq:sol_alpha} and \autoref{eq:w*}
	are used to compute $W^*$ for the kernel learned after the
	regular training of the neural network, which learn
	the embedding $\Phi_{L-1}$ and an estimate $W_L$~. In order
	to illustrate this result, and to compare the performances
	of the weights $W^*$ with respect to the weights $W_L$,
	learned either with usual learning strategies or
	with post-training, we train a neural network on two
	regression problems using a real and a synthetic data set.
	$70\%$ of the data are used for training, and $30 \%$ for testing.

\paragraph{Real Data Set Regression.}
For this experiment, we use the Parkinson Telemonitoring
data set \citep{tsanas2010accurate}. The input consists in
$5,875$ instances of $17$ dimensional data, and the output
are one dimensional real number. For this data set, a neural
network made of two fully connected hidden layers of size $17$
and $10$ with respectively \texttt{tanh} and RELU activation,
is trained for $250$, $500$ and $750$ iterations, with batches of
size $50$. The layer weights are all regularized with the $\ell_2$-norm
and a fixed regularization parameter $\lambda = 10^{-3}$ .
Then, starting from each of the trained networks, $200$ iterations of
post-training are used  with the same regularization parameter $\lambda$
and the performances are compared to the closed-form solutions
computed using \eqref{eq:sol_alpha} for each saved network.
The results are presented in \autoref{tab:regression}.

\paragraph{Simulated Data Set Regression.}
For this experiment, we use a synthetic data set.
The inputs were generated using a uniform distribution
on $\left[0,1\right]^{10}$. The outputs are computed
as follows:
\[
	Y = \tanh(XW_1)W_2
\]
where $W_1 \in \left[-1, 1\right]^{10\times 5}$ and
$W_2 \in \left[-1, 1\right]^{5}$ are randomly generated
using a uniform law. In total, the data set is composed of
$10,000$ pairs $(x_i, y_j).$
For this data set, a neural network with two fully connected
hidden layers of size $10$ with activation \texttt{tanh} for
the first layer and RELU for the second layer is trained for $250, 500$ and
$750$ iterations, with batches of size $50$. We use the same
protocol with $200$ extra post-training iterations.
The results are presented in \autoref{tab:regression}.

For these two experiments, the post-training improves the performances
toward these of the optimal solution, for several choices of stopping times.
It is worth noting that the performance of the optimal solution is better
when the first layers are not fully optimized with Parkinson Telemonitoring
data set. This effect denotes an overfitting tendency with the full training,
where the first layers become overly specified for the training set.

% subsection expopt (end)

\begin{center}
\begin{table}[t]
  \renewcommand{\arraystretch}{1.2}

\caption{Comparison of the performances (RMSE) of fully connected networks on different data sets, at different epochs, with or without post-training. }
\label{tab:regression}
\centering
\begin{tabular}{|c|c|c|c|c|}
\hline
Data set &Iterations&  \parbox{8em}{\vskip.3em\centering Error with \\classic training}&
					   \parbox{8em}{\vskip.3em\centering Error with \\ post-training} &
					   \parbox{8em}{\vskip.3em\centering Error with\\ optimal last layer }\\
\hline
\multirow{3}{*}{Parkinson} & 250 & 0.832 & 0.434 & 0.119 \\
& 500 & 0.147 & 0.147 & 0.140 \\
& 750 & 0.134 & 0.132 & 0.131 \\\hline
\multirow{3}{*}{Simulated} & 250 & 1.185 & 1.117 & 1.075 \\
& 500 & 0.533 & 0.450 & 0.447 \\
& 750 & 0.322 & 0.300 & 0.296 \\\hline

\end{tabular}
\end{table}

\end{center}

%%%

\crossref{}

\section{Discussion}
\label{sec:post_train:discussion}

% post-training is good
The experiments presented in \autoref{sec:post_train:numerical} show that
post-training improves the performances of all the networks
considered -- including recurrent, convolutional and fully
connected networks.
The gain is significant, regardless of the time at which the
regular training is stopped and the post-training is done.
In both the CIFAR10 and the PTB experiment, the gap between
the losses with and without post-training is more pronounced if the training is stopped early, and tends to be smaller as the network converges
to a better solution (see \autoref{fig:ptb} and \autoref{fig:cifar10}).
%
%When the network parameters reach an equilibrium point, the improvement tends to be smaller.
The reduction of the gap
between the test performances with and without post-training
is made clear in \autoref{tab:conv_xp}. For the MNIST data set,
with a small-size convolutional neural network, while the error rate
drops by 1.5\% when post-training is applied after 5000
iterations, this same error rate only drops by 0.2\% when it
is applied after 20000 iterations. This same observation
can be done for the other results reported in \autoref{tab:conv_xp}.
However, while the improvement is larger when the network did not
fully converge prior to the post-training, it is still significant
when the network has reached its minimum: for example in PTB the final
test perplexity is $81.7$ with post-training and  $82.4$ without; in
CIFAR10 the errors are respectively $0.147$ and $0.154$.

If the networks are allowed to moderately overfit, for instance by
training them with regular algorithm for a very large number of
iterations, the advantage provided by post-training vanishes:
for example in PTB the test perplexity after $2000$ iterations
(instead of $400$) is $83.2$ regardless of post-training.
This is coherent with the intuition behind the post-training:
after overfitting, the features learned by the network become
less appropriate to the general problem, therefore their optimal usage obtained by
post-training no longer provide an advantage.

%The post-training step is intended to ensure a quick convergence of the last layer weights to its optimal value when the previous layers are fixed. Indeed, when the network has already converged to a critical point, the post-training do not change much these weights as they are already nearly optimal, and the performance does not change.
%

%  This is inexpensive
It is important to note that the post-training computational
cost is very low compared to the full training computations.
For instance, in the CIFAR10 experiment, each iteration for
post-training is $4 \times$ faster on the same GPU than
an iteration using the full gradient. Also, in the different
experiments, post-training produces a performance gap after
using as little as $100$ batches.
There are multiple reasons behind this efficiency: first, the system
reaches a local minimum relatively rapidly for post-training as
the problem \autoref{eq:post_training} has a small number of
parameters compared to the dimensionality of the original
optimization problem. Second, the iterations used for
the resolution of \autoref{eq:post_training} are computationally
cheaper, as there is no need to chain high dimensional linear
operations, contrarily to regular backpropagation used during
the training phase. Finally, since the post-training optimization
problem is generally convex, the optimization is guaranteed
to converge rapidly to the optimal weights for the last layer.
%

% No overfitting
Another interesting point is that there is no evidence that the
post-training step leads to overfitting. In CIFAR10, the test
error is improved by the use of post-training, although the training
loss is similar. The other experiments do not show signs of
overfitting either as the test error is mostly improved by
our additional step.
This stems from the fact that the post-training
optimization is much simpler than the original problem as it
lies in a small-dimensional space -- which, combined with the
added $\ell_2$-regularization, efficiently prevents overfitting.
% Choice of parameters \lambda
The regularization parameter $\lambda$ plays an important role in post-training.
Setting $\lambda$ to be very large reduces the explanatory capacity
of the networks whereas if $\lambda$ is too small, the capacity
can become too large and lead to overfitting. Overall, our
experiments highlighted that the post-training produces significant
results for any choice of $\lambda$ reasonably small (i.e
\mbox{$10^{-5} \le \lambda \le 10^{-2}~$}). This parameter
is linked to the regularization parameter of the kernel
methods, as stated in \autoref{sec:post_train:kernel}.

Overall, these results show that the post-training step can be
applied to most trained networks, without prerequisites about how
optimized they are since post-training does not degrade their performances,
providing a consistent gain in performances for a very low
additional computational cost.

% Link to kernel
In \autoref{sub:expopt}, numerical experiments highlight the link
between post-training and kernel methods. As illustrated in
\autoref{tab:regression}, using the optimal weights derived from kernel
theory immediately a performance boost for the considered network.
The post-training step estimate numerically this optimal layer with
the gradient descent optimizer.
However, computing the optimal weights for the last layer is only achievable
for small data set due to the required matrix inversion. Moreover, the closed form
solution is known only for specific problems, such as kernelized least square regression. But post-training
approaches the same performance in these cases solving \autoref{eq:post_training}
with gradient-based methods.
%

% Pre training / post-training
The post-training can be linked very naturally to the idea of
pre-training, developed notably by
\citet{larochelle2007empirical}, \citet{Hinton2006} and \citet{hinton2006reducing}.
The unsupervised pre-training of a layer is designed to find a
representation that captures enough information from the data
to be able to reconstruct it from its embedding. The goal is
thus to find suitable parametrization of the general layers to extract
good features, summarizing the data. Conversely, the goal of the
post-training is, given a representation, to find the best
parametrization of the last layer to discriminate the data.
These two steps, in contrast with the usual training, focus
on respectively the general or specific layers.
%A possible extension of our work would be to use alternate minimization to come up with well matched parametrization of the layers.

\section{Conclusion}
\label{sec:ccl}

	In this work, we studied the concept of post-training,
	an additional step performed after the regular training,
	where only the last layer is trained. This step is intended
	to take fully advantage of the data representation learned
	by the network. We empirically shown that post-training is
	computationally inexpensive and provide a non negligible
	increase of performance on most neural network structures. 
	While we chose to focus on post-training solely the last
	layer -- as it is the most specific layer in the network
	and the resulting problem is strongly convex under
	reasonable prerequisites -- the relationship between the
	number of layers frozen in the post-training and the
	resulting improvements might be an interesting direction
	for future works.

% section ccl (end)

%%%


%\section*{Acknowledgment}
%the coffee machine

\newpage
\bibliographystyle{\Bibstyle}
\bibliography{\BibFile}

\newpage
\appendix

\title{Post Training in Deep Learning\\{\it Supplementary materials}}

\maketitle

\section{Convex loss}
\label{sec:post_train:appendix}

	We show here, for the sake of completeness, that the post-training problem is convex for
	the softmax activation in the last layer and the cross entropy loss. This result is
	proved showing that the hessian of the function is positive semidefinite, as it is
	a diagonally dominant matrix.

\begin{proposition}[convexity]
\label{prop:convex_cost}
$\forall N,M \in  \mathbb{N},$ $\forall X \in \Rset^{N}$, $\forall j \in \left[1,M \right]$, the following function F is convex:
\begin{align*}
F :& \Rset^{N\times M} \mapsto \Rset \\
&  W \rightarrow \log \left( \sum_{i=1}^M  \exp( XW_i) \right) - \sum_{i=1}^M  \delta_{ij} \log \left(  \exp(  XW_i)  \right)   .
\end{align*}
where $\delta$ is the Dirac function, and $W_i$ denotes the $i$-th row of a $W$.
\end{proposition}

\begin{proof}[Proof 1]
Let $$P_i(W) =  \frac{\exp(XW_i)}{\sum_{j=1}^M \exp(XW_j) }. $$

then
\begin{equation*}
\frac{\partial P_i}{\partial W_{m,n}} = \left\lbrace
\begin{aligned}
  & - x_{n} P_i(W) P_m(W)  \quad \text{ if $i \neq m$}\\
  & - x_{n} P_m^2(W) + x_n P_m(W) \quad \text{otherwise}
\end{aligned}
\right.
\end{equation*}
Noting that
$$ F(W) = -\sum_{i=1}^M  \delta_{ij} \log \left( P_i(W) \right),$$
we have
\begin{align*}
&\frac{\partial F(W)}{\partial W_{m,n}} = - \sum_{i=1}^M  \delta_{ij} \frac{1}{P_i(W)} \frac{\partial P_i}{\partial W_{m,n}} \\
\quad &= x_n \left(  \sum_{i=1}^M  \delta_{ij} P_m(W)  - \delta_{mj}   \right)\\
\quad &= x_n \left( P_m(W)  - \delta_{mj}   \right),
\end{align*}
hence
\begin{align*}
&\frac{\partial^2 F(W)}{\partial W_{m,n}\partial W_{p,q}} =   x_n \left(   \frac{\partial P_m}{\partial W_{p,q}} \right),\\
& \quad =  x_n x_q  P_m(W) \left( \delta_{m,p} - P_p(W)  \right).
\end{align*}
Hence the following identity
$$ H(F) =\mathbf{P}(W)  \otimes  (X X\tran)$$
where $\otimes$ is the Kronecker product, and the matrix $\mathbf{P}(W)$ is defined by $\mathbf{P} _{m,p} = P_m(W) \left( \delta_{m,p} - P_p(W) \right)$. Now since  $\forall 1 \le m \le M$,

\begin{align*}
&\sum_{p=1,p \neq m}^M \left\vert \mathbf{P}_{m,p} \right\vert=   P_m(W) \sum_{p=1,p \neq m}^M  P_p(W) \\
& \quad =  P_m(W) \left( 1 - P_m(W) \right)\\
& \quad = \mathbf{P}_{m,m}
\end{align*}
$\mathbf P(W)$ is thus a diagonally dominant matrix. Its diagonal elements are positive
\[
	\mathbf{P}_{m, m} = P_m(W) \left( 1 - P_m(W) \right) \ge 0, ~~~~ \text{ as } P_m(W) \in [0, 1]
\]
and thus $\mathbf P(W)$ is positive semidefinite. Since $XX\tran$ is positive semidefinite too,
their Kronecker product is also positive semidefinite, hence the conclusion.

\end{proof}

\end{document}